\documentclass[twoside]{article}

\usepackage[accepted]{aistats2022}
\usepackage{url}            
\usepackage{booktabs}       
\usepackage{amsfonts}       
\usepackage{xcolor}         
\usepackage{tabularx}
\usepackage{tikz-cd}        
\usepackage{algorithm}
\usepackage{algorithmic}
\usepackage{multicol}
\usepackage{subcaption}
\usepackage[export]{adjustbox}

\usepackage{amsthm}
\usepackage{thmtools, thm-restate}
\declaretheorem[numberwithin=section]{thm}
\declaretheorem[sibling=thm]{lemma}

\declaretheorem[sibling=thm]{definition}

\usepackage{xspace}
\DeclareRobustCommand{\eg}{e.g.,\@\xspace}
\DeclareRobustCommand{\ie}{i.e.,\@\xspace}
\DeclareRobustCommand{\wrt}{w.r.t.\@\xspace}
\usepackage{amsmath}
\usepackage{amssymb}
\usepackage{nicefrac}
\usepackage{dsfont}
\usepackage{bm}

\newcommand{\de}{\,\mathrm{d}}
\newcommand{\Aspace}{\mathcal{A}}
\newcommand{\Sspace}{\mathcal{S}}
\newcommand{\R}{\mathcal{R}}
\DeclareMathOperator*{\Rmax}{R_{max}}
\newcommand{\cmp}{\mathcal{M}}
\newcommand{\mdp}{\mathcal{M}^\R}

\newcommand{\vtheta}{\bm{\theta}}
\newcommand{\vmu}{\bm{\mu}}
\DeclareMathOperator*{\EV}{\mathbb{E}}
\DeclareMathOperator*{\Var}{\mathbb{V}ar}
\newcommand{\Reals}{\mathbb{R}}
\DeclareMathOperator*{\argmax}{arg\,max}
\DeclareMathOperator*{\argmin}{arg\,min}
\newcommand{\cTheta}{\Theta_\sigma}

\usepackage{hyperref}

\usepackage[round]{natbib}

\renewcommand{\cite}{\citep}

\begin{document}

%

%
\runningauthor{Mirco Mutti, Stefano Del Col, Marcello Restelli}

\twocolumn[

\aistatstitle{Reward-Free Policy Space Compression for Reinforcement Learning}

\aistatsauthor{Mirco Mutti\textsuperscript{*} \And Stefano Del Col \And  Marcello Restelli}

\aistatsaddress{Politecnico di Milano \And  Politecnico di Milano \And Politecnico di Milano \AND Universit\`a di Bologna \And \And } 

]

\begin{abstract}
In reinforcement learning, we encode the potential behaviors of an agent interacting with an environment into an infinite set of policies, the \emph{policy space}, typically represented by a family of parametric functions.
Dealing with such a policy space is a hefty challenge, which often causes sample and computation inefficiencies. However, we argue that a limited number of policies are actually relevant when we also account for the structure of the environment and of the policy parameterization, as many of them would induce very similar interactions, \ie state-action distributions.
In this paper, we seek for a reward-free \emph{compression} of the policy space into a finite set of representative policies, such that, given any policy $\pi$, the minimum R\'enyi divergence between the state-action distributions of the representative policies and the state-action distribution of $\pi$ is bounded.
We show that this compression of the policy space can be formulated as a set cover problem, and it is inherently NP-hard. Nonetheless, we propose a game-theoretic reformulation for which a locally optimal solution can be efficiently found by iteratively stretching the compressed space to cover an adversarial policy. 
Finally, we provide an empirical evaluation to illustrate the compression procedure in simple domains, and its ripple effects in reinforcement learning.
\end{abstract}

\section{INTRODUCTION}
In the Reinforcement Learning~(RL)~\cite{sutton2018reinforcement} framework, an artificial agent interacts with an environment, typically modeled through a Markov Decision Process~(MDP)~\cite{puterman2014markov}, to maximize some form of long-term performance, which is usually the sum of the discounted rewards collected in the process.
The agent's behavior is encoded in a Markovian \emph{policy}, \ie a function that maps the current state of the environment with a probability distribution over the next action to be taken. In principle, if the underlying MDP is small enough, we can represent a Markovian policy with a table that includes an entry for each state-action pair, and we call it a \emph{tabular} policy. However, most relevant scenarios have too many (possibly infinite) states and actions to allow for a tabular representation. In this case, we can turn to function approximation~\cite{sutton2018reinforcement} to encode the policy within a family of parametric functions, \eg a linear basis combination or a deep neural network, and we call it a \emph{parametric} policy. This set of parametric policies, which we call the \emph{policy space}, is typically infinite. Therefore, learning a policy that maximizes the performance can be a hefty challenge, and the sheer size of the policy space often causes sample and computation inefficiencies.

A setting where these inefficiencies arise clearly and naturally is Policy Optimization~(PO)~\cite{deisenroth2013survey}. In PO, we aim to find a policy that maximizes the performance within the policy space, \ie an \emph{optimal} policy, with the least amount of interactions~\cite{sutton1999pg, silver2014deterministic, schulman2015trpo, metelli2018pois}. If we also account for the performance of the policies that are actually deployed to collect these interactions, we come up with an online PO~\cite{papini2019optimist, cai2020oppo}. In this setting, we try to minimize the \emph{regret} that the agent suffers by taking interactions with a sub-optimal behavior before converging to an optimal policy. Recent results showed that the regret of online PO is directly related to the size of the policy space~\cite{papini2019optimist, metelli2020randomist}. In particular, online PO with a finite policy space can enjoy a constant regret, \ie it does not scale with the number of interactions, under certain conditions~\cite{metelli2020randomist}. Instead, the regret of online PO with an infinite policy space does scale with the square root of the number of interactions in general~\cite{papini2019optimist}, which means that we only have  asymptotic guarantees of reaching an optimal policy. In view of these results, one could wonder whether the expressive power of an infinite policy space is worth the additional regret it causes: Are all of these infinitely many policies really necessary for PO? The expressive power of a policy space is related to the different distributions that its policies can induce over the states and actions of the environment, as the whole point of PO is to find a policy that maximizes the probability of reaching state-action pairs associated with high rewards. However, different parameterizations might actually induce equivalent policies due to the specific structure of the policy space. Similarly, even different policies can induce the same state-action distribution in a given environment. These two types of policies are arguably redundant for PO and we would like to find a policy space that does not include either. Especially, we aim to answer the following question:
\begin{center}
\textit{  Having an infinite parametric policy space $\Theta$ in a given environment $\cmp$, can we compress $\Theta$ into a finite subset that retains most of its expressive power?}
\end{center}
In this paper, we formulate this question into the \emph{Policy Space Compression} problem, where we exploit the inherent structure of $\cmp$ and $\Theta$ to compute the compressed policy space. The general idea is to identify a finite set of representative policies, such that for any policy $\pi$ of the original space, the minimum R\'enyi divergence between the state-action distributions of the representative policies and the state-action distribution of $\pi$ is bounded by a given constant. This compression is agnostic to the reward function, and thus the resulting policy space can benefit the computational and sample complexity of any RL task one can later specify over $\cmp$, as it is typical in reward-free RL~\cite{hazan2019maxent, jin2020rf}.

Specifically, the paper includes the following contributions. 
First, we provide a formal definition of the policy space compression problem (Section~\ref{sec:problem}). We note that the problem can be formulated equivalently as a set cover, and that finding an optimal compression of the policy space is NP-hard in general~\cite{feige1998threshold}. 
Despite this negative result, we propose a game-theoretic reformulation (Section~\ref{sec:game}) that casts the problem to the one of reaching a differential Stackelberg equilibrium~\cite{fiez2020stackelberg} of a two-player sequential game, in which the first player tries to cover the policy space with a finite set of policies and the second player tries to find a policy that falls outside this coverage. 
Then, we present a planning algorithm (Section~\ref{sec:algorithm}) to efficiently compute a compression of the policy space in a given environment, by repeatedly solving, with a first-order method, the two-player game for an increasing number of covering policies, until the compression requirement is met globally. 
In Section~\ref{sec:theory}, we provide a theoretical analysis of the performance guarantees attained by the compressed policy space in relevant RL tasks, namely policy evaluation and policy optimization.
Finally, in Section~\ref{sec:numerical_validation} we provide a brief numerical validation of both the compression algorithm and RL with the compressed policy space. The proofs of the theorems can be found in Appendix~\ref{apx:proofs}.

\section{PRELIMINARIES}
\label{sec:preliminaries}
In this section, we introduce the essential background on controlled Markov processes, policy optimization, importance sampling estimation and R\'enyi divergence. 
Throughout the paper, we will denote a vector $\bm{v}$ with a bold typeface, as opposed to a scalar $v$.

\subsection{Controlled Markov Processes}

A discrete-time Controlled Markov Process (CMP) is defined as a tuple $\cmp := (\Sspace, \Aspace, P, \mu, \gamma)$, in which $\Sspace$ is the state space, $\Aspace$ is the action space, $P : \Sspace \times \Aspace \to \Delta (\Sspace)$ is a transition model such that the next state is drawn as $s' \sim P(\cdot |s,a)$ given the current state $s \in \Sspace$ and action $a \in \Aspace$, $\mu: \Delta (\Sspace)$ is an initial state distribution such that the initial state is drawn as $s \sim \mu(\cdot)$, and $\gamma \in [0, 1]$ is the discount factor.
The behavior of an agent interacting with a CMP can be modeled through a Markovian parametric policy $\pi_{\vtheta} : \Sspace \to \Delta (\Aspace)$ such that an action is drawn as $a \sim \pi_{\vtheta} (\cdot| s)$ given the current state $s \in \Sspace$, where $\vtheta \in \Theta\subseteq \Reals^{m}$ are the policy parameters, and the set $\Pi_\Theta$ is called the \emph{policy space}.
A policy $\pi_{\vtheta}$ induces a $\gamma$-discounted state distribution $d_{\pi_{\vtheta}}^s : \Delta(\Sspace)$ over the state space of the CMP $\cmp$, which is given by $d_{\pi_{\vtheta}}^s (s) = (1 - \gamma) \sum_{t = 1}^{\infty} \gamma^t Pr(s_t = s)$ or the equivalent recursive relation $d_{\pi_{\vtheta}}^s (s) = (1 - \gamma) \mu (s) - \gamma \int_{\Sspace\Aspace} d_{\pi_{\vtheta}}^s (s') \pi_{\vtheta} (a' | s') P(s | s', a') \de s' \de a'$. Similarly, we define the $\gamma$-discounted state-action distribution $d_{\pi_{\vtheta}}^{sa} : \Delta(\Sspace \times \Aspace)$ given by $d_{\pi_{\vtheta}}^{sa} (s, a) = \pi_{\vtheta} (a | s) d_{\pi_{\vtheta}}^s (s)$. With a slight overloading of notation, we will indifferently denote the parametric policy space $\Pi_{\Theta}$ by $\Theta$, a parametric policy $\pi_{\vtheta} \in \Pi_\Theta$ by $\vtheta$, and its induced distributions $d^s_{\pi_{\vtheta}} (s), d^{sa}_{\pi_{\vtheta}} (s,a)$ by $d_{\vtheta}^s (s), d_{\vtheta}^{sa} (s, a)$.

\subsection{Policy Optimization}
\label{sec:preliminaries_policy_optimization}
The process of looking for the policy that maximizes the agent's performance on a given RL task with a direct search in the policy space is called Policy Optimization (PO)~\cite{deisenroth2013survey}. The task is generally modeled through a Markov Decision Process (MDP)~\cite{puterman2014markov} $\mdp := \cmp \cup \R$, \ie the combination of a CMP $\cmp$ and a reward function $\R : \Sspace \times \Aspace \to [- \Rmax, \Rmax]$ such that $R(s, a)$ is the bounded reward that the agent collects by selecting action $a \in \Aspace$ in state $s \in \Sspace$, and $\Rmax < \infty$. The agent's performance is defined by the expected sum of discounted rewards collected by its policy, \ie
\begin{align*}
	J (\vtheta) &:= \EV_{\substack{s_0 \sim \mu (\cdot), a_t \sim \pi_{\vtheta} (\cdot | s_t) \\ s_{t + 1} \sim P (\cdot | s_t, a_t) }} \bigg[  \sum_{t = 1}^\infty \gamma^t \R (s_t, a_t) \bigg] \\
	&= \frac{1}{(1 - \gamma)} \EV_{(s, a) \sim d^{sa}_{\vtheta}} \big[ \R (s, a) \big],
\end{align*}
A Monte-Carlo estimate of the performance can be computed from a batch of $N$ samples $\{ s_n, a_n \}_{n = 1}^N$ taken with the policy $\pi_{\vtheta}$ in the $\gamma$-discounted MDP $\mdp$ as $\widehat{J} (\vtheta) = \frac{1}{(1 - \gamma) N} \sum_{n = 1}^{N} \R (s_n, a_n)$.

\subsection{Importance Sampling and R\'enyi Divergence}

Importance Sampling (IS)~\cite{cochran2007sampling, owen2013monte} is a common technique to estimate the expectation of a function under a \emph{target} distribution by taking samples from a different distribution. In PO, importance sampling allows for estimating the performance of a target policy $\pi_{\vtheta'}$ through a batch of samples $\{ s_n, a_n \}_{n = 1}^N$ taken with a policy $\pi_{\vtheta}$. Especially, we define the importance weight $w_{\vtheta' / \vtheta} (s, a) := d_{\vtheta'}^{sa} (s, a) / d_{\vtheta}^{sa} (s, a)$. A Monte-Carlo estimate of $J (\vtheta')$ via importance sampling is given by
\begin{equation*}
	\widehat{J}_{IS} (\vtheta' / \vtheta) = \frac{1}{(1 - \gamma) N} \sum_{n = 1}^{N} w_{\vtheta' / \vtheta} (s_n, a_n) \R (s_n, a_n).
\end{equation*}
The latter estimator is known to be unbiased, \ie $\EV_{\vtheta} [\widehat{J}_{IS} (\vtheta' / \vtheta)] = J (\vtheta')$~\cite{owen2013monte}.
However, $\widehat{J}_{IS} (\vtheta' / \vtheta)$ might suffer from a large variance whenever the importance weights $w_{\vtheta' / \vtheta} (s, a)$ have a large variance. The variance of the importance weights is related to the exponentiated 2-R\'enyi divergence $D_2 (d_{\vtheta'}^{sa} || d_{\vtheta}^{sa})$~\cite{renyi1961measures} through $\Var_{(s,a) \sim d^{sa}_{\vtheta}} [ w_{\vtheta' / \vtheta} (s,a) ] = D_2 (d_{\vtheta'}^{sa} || d_{\vtheta}^{sa}) - 1$~\cite{cortes2010learning}, where
\begin{equation*}
	D_2 (d_{\vtheta'}^{sa} || d_{\vtheta}^{sa}) := \int_{\mathcal{SA}} d_{\vtheta}^{sa} (s, a) \bigg( \frac{d_{\vtheta'}^{sa} (s,a) }{d_{\vtheta}^{sa} (s,a) }  \bigg)^2 \de s \de a.
\end{equation*}
The latter has been employed in~\cite{metelli2018pois} to upper bound the variance of the importance sampling estimator as $\Var_{(s,a) \sim d_{\vtheta}^{sa}} [ \widehat{J}_{IS} (\vtheta' / \vtheta) ] \leq \big( \frac{\Rmax}{1 - \gamma} \big)^2 D_2 (d_{\vtheta'}^{sa} || d_{\vtheta}^{sa}) / N$. In the following, we will refer to the exponentiated 2-R\'enyi divergence as the R\'enyi divergence.

\section{THE POLICY SPACE COMPRESSION PROBLEM}
\label{sec:problem}
Let us suppose to have a CMP $\cmp$ the agent can interact with, and a parametric policy space $\Theta$ from which the agent can select its strategy of interaction.
For the common parameterization choices, ranging from linear policies to deep neural networks, the policy space $\Theta$ is typically infinite.
Dealing with such a large policy space to address the usual RL tasks, \eg finding a convenient task-agnostic sampling strategy~\cite{hazan2019maxent} or seeking for an optimal policy within the set~\cite{deisenroth2013survey}, is often a huge challenge.
Furthermore, many policies in $\Theta$ are unnecessary for these purposes, as they induce very similar interactions, and thus they have very similar performance.
On the one hand, different policy parameters $\vtheta \in \Theta$ might induce nearly identical distributions over actions. On the other hand, even different distributions over actions can lead to comparable state-action distributions due to the structure of the environment. Since we do not have any reward encoded in $\cmp$, it would be unwise to deem any state-action distribution irrelevant without additional information on the task structure.
In this work, we aim to identify a subset of the policy space $\Theta' \subseteq \Theta$ that retains most of the expressive power of $\Theta$, \ie the set of the state-action distributions it can induce, while dramatically reducing its size, to the advantage of the computational and sample efficiency of future RL tasks.
Especially, we consider a $\sigma$-soft compression of $\Theta$, where for any policy $\vtheta \in \Theta$ we would like to have a policy $\vtheta' \in \Theta'$ such that the R\'enyi divergence between their respective state-action distributions $d_{\vtheta}^{sa}, d_{\vtheta'}^{sa}$ is bounded by a positive constant $\sigma$. The R\'enyi divergence is particularly convenient in this setting due to its relationship with the variance of the importance sampling in the off-policy estimation~\cite{cortes2010learning, metelli2018pois}. The following statement provides a more formal definition of this $\sigma$-soft compression.
\begin{definition}[$\sigma$-compression]
	\label{thr:sigma_compression}
	Let $\cmp$ be a CMP, let $\Theta$ be a parametric policy space for $\cmp$, and let $\sigma > 0$ be a constant. We call $\cTheta$ a $\sigma$-compression of $\Theta$ in $\cmp$ if it holds that $| \cTheta | < \infty$ and
	\begin{equation*}
	\begin{aligned}
		&\forall \vtheta \in \Theta, 
		&\min_{\vtheta' \in \cTheta} D_2 (d_{\vtheta}^{sa} || d_{\vtheta'}^{sa}) \leq \sigma.
	\end{aligned}
	\end{equation*}
\end{definition}
We call the task of finding a $\sigma$-compression of $\Theta$ in $\cmp$ the \emph{policy space compression} problem. Notably, for some $\cmp, \Theta, \sigma$, a $\sigma$-compression of $\Theta$ in $\cmp$ might not exist, as infinitely many policies $\vtheta \in \Theta$ might induce relevant state-action distributions. However, we note that those scenarios are not interesting for our purposes, as the PO problem would be far-fetched as well, since one should try infinitely many policies to find an optimal policy. Instead, we only consider scenarios in which the $\sigma$-compression is \emph{feasible}.
In these cases, given $\cmp$ and $\Theta$, we would like to extract the smallest set of policies $\Theta'$ that is a $\sigma$-compression of $\Theta$ in $\cmp$, and then keep this reduced policy space to address any RL task one can define over $\cmp$. 
Let $\Omega_\Theta := \{ d_{\vtheta}^{sa} \ | \ \forall \vtheta \in \Theta \}$ be the set of state-action distributions induced by the policy space $\Theta$, the compression problem can be formulated as a typical \emph{set cover problem}, \ie
\begin{equation}
\begin{aligned}
    &\text{minimize}
    & & \sum_{\omega \in \Omega_\Theta} x_{\omega} \\
    & \text{subject to} 
    & & \sum_{\omega : D_2 (\upsilon || \omega ) \leq \sigma } x_{\omega} \geq 1, \quad \forall \upsilon \in \Omega_\Theta \\
    & & & x_{\omega} \in \{0, 1\}, \quad \forall \omega \in \Omega_\Theta
\end{aligned}
	\label{eq:set_cover}
\end{equation}
where the positive integers $x_{\omega}$ denote the state-action distributions that are active in the covering, and the corresponding $\sigma$-compression of $\Theta$ in $\cmp$ can be retrieved as $\cTheta = \{ \vtheta \in \Theta \ | \ d_{\vtheta}^{sa} = \omega \ \wedge \ x_\omega = 1  \}$. Unfortunately, the problem~\eqref{eq:set_cover} is known to be NP-hard~\cite{feige1998threshold}, even when the model of $\cmp$ is fully available. Two aspects arguably make this problem extremely hard: On the one hand, we are looking for an efficient solution in the number of active state-action distributions, secondly, we are covering the set $\Omega_\Theta$ all at once rather than incrementally. Instead of considering common relaxations of~\eqref{eq:set_cover}~\cite{johnson1974approximation, lovasz1975ratio}, which would not strictly meet the requirements of Definition~\ref{thr:sigma_compression}~\cite{feige1998threshold}, in the next section we build on these insights to reformulate the policy space compression problem in a tractable way.

\section{A GAME THEORETIC REFORMULATION}
\label{sec:game}
Due to its inherent hardness, we aim to find a tractable reformulation of the policy space compression problem~\eqref{eq:set_cover} whose solution is a valid $\sigma$-compression of $\Theta$ in $\cmp$. Let us consider a game-theoretic perspective to the set cover problem. A first player distributes a set of $K$ policies $(\vtheta_1, \ldots, \vtheta_K) \in \Theta^K$ with the intention of covering the set of state-action distributions $\Omega_\Theta$. A second player tries to find a policy $\vmu \in \Theta$ that is not well covered by $(\vtheta_1, \ldots, \vtheta_K)$, \ie a policy that maximizes the R\'enyi divergence between its state-action distribution and the one of the closest $\vtheta_k \in (\vtheta_1, \ldots, \vtheta_K)$. The former player moves first, and we call it a \emph{leader}. The latter player makes his move in response to the other player, and it is then called a \emph{follower}. The two-player, zero-sum, sequential game that we have informally described can be represented as the optimization problem
\begin{gather}
	\min_{\vtheta \in \Theta^K} \max_{\vmu \in \Theta} \ f(\vtheta, \vmu), \label{eq:game} \\
	f(\vtheta, \vmu) := \min_{k \in [K]} D_2 (d_{\vmu}^{sa} || d_{\vtheta_k}^{sa}), \nonumber
\end{gather}
where $\vtheta = (\vtheta_1, \ldots, \vtheta_K)$ and $[K] = \{ 1, \ldots, K \}$.
It is straightforward to see that if the $\sigma$-compression is feasible for $\Theta$ in $\cmp$ and $K$ is large enough, then any optimal leader's strategy for the game~\eqref{eq:game}, \ie $\vtheta^* \in \argmin_{\vtheta \in \Theta^K} \max_{\vmu \in \Theta} f(\vtheta, \vmu)$, is a $\sigma$-compression of $\Theta$ in $\cmp$. Unfortunately, $f(\vtheta, \vmu)$ is a non-convex non-concave function, and finding a globally optimal strategy for the game~\eqref{eq:game} is still a NP-hard problem. However, we do not actually need to find a globally optimal strategy for the leader, as any $\vtheta \in \Theta^K$ such that $\min_{\vmu \in \Theta} f(\vtheta, \vmu) \leq \sigma$ would be a valid $\sigma$-compression of $\Theta$. Thus, we might instead target a locally optimal strategy for~\eqref{eq:game}, which is a stationary point of $f$ that is both a local maximum \wrt $\vtheta$ and a local minimum \wrt $\vmu$. We formalize this solution concept as a Differential Stackelberg Equilibrium (DSE)~\cite{fiez2020stackelberg}.
\begin{definition}[Differential Stackelberg~\cite{fiez2020stackelberg}]
	The joint strategy $(\vtheta^*, \vmu^*) \in \Theta^{K + 1}$ in which $\vtheta^*_k \in \argmin_{k \in [K]} (d_{\vmu^*}^{sa} || d_{\vtheta^*_k}^{sa})$ is a differential Stackelberg equilibrium of the game~\eqref{eq:game} if it holds $\nabla_{\vtheta^*_k} f (\vtheta^*, \vmu^*) = 0, \nabla_{\vmu^*} f (\vtheta^*, \vmu^*) = 0, | \nabla_{\vtheta_k^*} \nabla_{\vtheta_k^*}^{\top} f (\vtheta^*, \vmu^*) | > 0$, and $| \nabla_{\vmu^*} \nabla_{\vmu^*}^{\top} f(\vtheta^*, \vmu^*) | < 0$.
	\footnote{Let $f (\bm{x})$ be a function of $\bm{x} \in \Reals^m$, we denote its gradient vector as $\nabla_{\bm{x}} f(\bm{x})$, its Hessian matrix as $\nabla_{\bm{x}} \nabla_{\bm{x}}^\top f (\bm{x})$, and the determinant of its Hessian matrix as $| \nabla_{\bm{x}} \nabla_{\bm{x}}^\top f (\bm{x}) |$.}
\end{definition}
Luckily, several recent works have established a favorable complexity for the problem of finding a DSE~\cite{jin2020local,fiez2020stackelberg,fiez2020tau} in a sequential game. Especially, \citet{jin2020local} showed that a basic first-order method, \ie Gradient Descent Ascent (GDA), with an infinite time-scale separation between the leader's and follower's updates is guaranteed to converge to a DSE under mild conditions.
This result might be surprising, as we started with a fundamentally hard problem~\eqref{eq:set_cover} and ended up with a way easier formulation~\eqref{eq:game} that we can address with a common methodology, without making any strong assumption on the structure of the problem.
However, we still have to deal with two crucial issues to solve the policy space compression problem through the game-theoretic formulation. On the one hand, it is not enough to look at the value $f(\vtheta^*, \vmu^*)$ attained by a DSE $(\vtheta^*, \vmu^*)$ to guarantee that $\vtheta$ is a $\sigma$-compression of $\Theta$, as we should check that $\max_{\vmu \in \Theta} f(\vtheta^*, \vmu) \leq \sigma$, where $\vmu$ is a global maximizer. On the other hand, it is not clear how to set a convenient value of $K$ beforehand.
In the next section, we present a first-order method that addresses these two issues by finding a DSE of iteratively larger instances of the game~\eqref{eq:game} (which we will henceforth call the \emph{cover game}) until a conservative approximation of the global condition $\max_{\vmu \in \Theta} f(\vtheta^*, \vmu) \leq \sigma$ is finally met.
%
%
%
%
%

\section{A PLANNING ALGORITHM TO SOLVE THE PROBLEM}
\label{sec:algorithm}
Optimization problems of the kind of~\eqref{eq:game} are typically addressed with a GDA procedure, in which the leader's parameters ($\vtheta$) and the follower's parameters ($\vmu$) are updated iteratively according to
\begin{equation*}
	\vtheta \gets \vtheta - \alpha \nabla_{\vtheta} f(\vtheta, \vmu), \qquad \vmu \gets \vmu + \beta \nabla_{\vmu} f(\vtheta, \vmu),
\end{equation*}
where $\nabla_{\vtheta} f(\vtheta, \vmu)$ and $\nabla_{\vmu} f(\vtheta, \vmu)$ are the respective gradients of the joint objective function, $\alpha > 0$ and $\beta > 0$ are learning rates. Especially, if we consider a sufficiently large time-scale separation $\tau := \beta / \alpha$, we are guaranteed to converge to a DSE of the game~\eqref{eq:game}~\cite{jin2020local, fiez2020tau}. In this case, we can consider $\tau = \infty$, which means that we update the follower's parameters until a stationary point is reached, \ie $\nabla_{\vmu} f(\vtheta, \vmu) = 0$, before updating the leader's parameters.
However, to instantiate the cover game, we still need to specify the number $K$ of leader-controlled policies $\vtheta = (\vtheta_1, \ldots, \vtheta_K)$. A straightforward solution is to start with a small number of policies first, say $K=1$, then retrieve a DSE $(\vtheta^*, \vmu^*)$ via GDA for a cover-game instance with $K$ policies, and finally check if the resulting leader's strategy $\vtheta^*$ meets the global requirement $\max_{\vmu \in \Theta} f (\vtheta^*, \vmu) \leq \sigma$. If the answer is positive, the policy space compression problem is solved, and $\vtheta^*$ is a $\sigma$-compression of $\Theta$ in $\cmp$. Otherwise, we increment $K$ and we repeat the process to see if we can solve the problem with more policies in $\vtheta$. If the policy space compression problem is feasible, with this simple procedure we are guaranteed to get a valid $\sigma$-compression eventually.
We call this method the \emph{Policy Space Compression Algorithm} (PSCA) and we report the pseudocode in Algorithm~\ref{alg:psca}. In the following sections, we describe in details how the optimization of the follower's parameters (Section~\ref{sec:algorithm_follower}) and the leader's parameters (Section~\ref{sec:algorithm_leader}) are carried out in an adaptation of the GDA method to the specific setting of the cover game. In Section~\ref{sec:algorithm_guarantee}, we discuss how to verify the global requirement $\max_{\vmu \in \Theta} f (\vtheta^*, \vmu) \leq \sigma$ without actually having to find a globally optimal follower's strategy, but instead optimizing a surrogate objective through a tractable linear program.
\begin{algorithm}[t!]
    \caption{PSCA}
    \label{alg:psca}
    \begin{flushleft}
    \begin{algorithmic}[t]
        \STATE \textbf{Input}: CMP $\cmp$, policy space $\Theta$, constant $\sigma$
        \STATE initialize $K = 0$ and the cover guarantee $\overline{\mathcal{Z}}_{\vtheta} = \infty$
        \WHILE{ $(\overline{\mathcal{Z}}_{\vtheta})^2 > \sigma$ }
        \STATE $K \gets K + 1$
        \STATE initialize the leader $\vtheta = (\vtheta_1, \ldots, \vtheta_K) \in \Theta^K$
        \FOR{epoch = $1, 2, \ldots, $ until convergence}
            \STATE compute the best response $\vmu_{br}$ to $\vtheta$
            \STATE identify the active leader's component $\vtheta_k$
            \STATE update the leader $\vtheta_k \gets \vtheta_k - \alpha \nabla_{\vtheta_k} f(\vtheta, \vmu_{br})$
        \ENDFOR
        \STATE compute the cover guarantee $\overline{\mathcal{Z}}_{\vtheta}$ with~\eqref{eq:lp_guarantee}
        \ENDWHILE
        \STATE \textbf{Output}: return $\vtheta$, a $\sigma$-compression of $\Theta$ in $\cmp$
    \end{algorithmic}
    \end{flushleft}
\end{algorithm}

\subsection{Optimizing the Follower's Parameters}
\label{sec:algorithm_follower}

In principle, we would like to compute the gradient $\nabla_{\vmu} f(\vtheta, \vmu)$ to perform the update $\vmu \gets \vmu + \beta \nabla_{\vmu} f(\vtheta, \vmu)$ as in a common GDA procedure. Unfortunately, the objective function $f(\vtheta, \vmu) = \min_{k \in [K]} D_2 (d_{\vmu}^{sa} || d_{\vtheta_k}^{sa})$ is not differentiable due to the minimum over the $K$ components of $\vtheta$. However, only the leader's component $\vtheta_k$ that attains the minimum of $f$ is actually relevant for the follower's update, as the other $K - 1$ components do not affect the value of the objective. Thus, we call $\vtheta_k \in \argmin_{\vtheta_i \in \vtheta} D_2 (d_{\vmu}^{sa} || d_{\vtheta_k}^{sa})$ the \emph{active leader's component}. Conveniently, we can update the follower's parameters \wrt the gradient $\nabla_{\vmu} f(\vtheta_k, \vmu)$, which is differentiable \wrt $\vmu$. The following proposition provides the formula for this gradient.
\begin{restatable}[Follower's Gradient]{proposition}{followerGradient}
	Let $(\vtheta, \vmu) \in \Theta^K$, the gradient of $f(\vtheta, \vmu)$ \wrt $\vmu$ is given by
	\begin{align}
		&\nabla_{\vmu} f (\vtheta, \vmu) = \nonumber \\
		&2 \EV_{(s,a) \sim d_{\vtheta_k}^{sa}} \bigg[ \bigg( \frac{d_{\vmu}^{sa} (s,a)}{d_{\vtheta_k}^{sa} (s,a)} \bigg)^2 \ \nabla_{\vmu} \log d_{\vmu}^{sa} (s, a)  \bigg],
	\label{eq:follower_gradient}
	\end{align}
	where $\vtheta_k$ is the active leader's component such that $\vtheta_k \in \argmin_{\vtheta_i \in \vtheta} D_2 (d_{\vmu}^{sa} || d_{\vtheta_i}^{sa} )$.
\end{restatable}
To perform a full optimization of the follower's parameters, we just need to repeatedly apply the gradient ascent update with the gradient $\nabla_{\vmu} f(\vtheta, \vmu)$ computed as in~\eqref{eq:follower_gradient}. Under mild conditions on the learning rate~\cite{robbins1951stochastic}, this process is guaranteed to converge to a stationary point such that $\nabla_{\vmu} f(\vtheta, \vmu) = 0$. We call the follower's parameters $\vmu$ at this stationary point the \emph{best response} to the leader's parameter $\vtheta$, and we denote it as $\vmu_{br}$.

\subsection{Optimizing the Leader's Parameters}
\label{sec:algorithm_leader}

Whenever the follower converges at the best response $\vmu_{br}$ to the current leader's parameters, we would like to make an update to $\vtheta$ in the direction of the gradient $\nabla_{\vtheta} f(\vtheta, \vmu)$, \ie $\vtheta \gets \vtheta - \alpha \nabla_{\vtheta} f(\vtheta, \vmu)$. Just as before, we can pre-compute the active leader's component $\vtheta_k \in \argmin_{\vtheta_i \in \vtheta} D_2 (d_{\vmu}^{sa} || d_{\vtheta_i}^{sa} )$ to make an update to $\vtheta_k$ in the direction of the gradient $\nabla_{\vtheta_k} f(\vtheta_k, \vmu)$, which is differentiable in $\vtheta_k$. Indeed, an update to any other leader's component would not have a meaningful impact on the value of the objective, whereas updating $\vtheta_k$ with a sufficiently small learning rate $\alpha$ is guaranteed to decrease $f (\vtheta, \vmu)$, possibly forcing the follower to change its best response in the next epoch. The following proposition provides the formula for the gradient.
\begin{restatable}[Leader's Gradient]{proposition}{leaderGradient}
	Let $(\vtheta, \vmu) \in \Theta^K$, the gradient of $f(\vtheta, \vmu)$ \wrt $\vtheta_k$ is given by
	\begin{align}
		&\nabla_{\vtheta_k} f (\vtheta, \vmu) = \nonumber \\
		&- \EV_{(s,a) \sim d_{\vtheta_k}^{sa}} \bigg[ \bigg( \frac{d_{\vmu}^{sa} (s,a)}{d_{\vtheta_k}^{sa} (s,a)} \bigg)^2 \ \nabla_{\vtheta_k} \log d_{\vtheta_k}^{sa} (s, a)  \bigg].
	\label{eq:leader_gradient}
	\end{align}
\end{restatable}

\subsection{Assessing the Global Value of the Leader's Parameters}
\label{sec:algorithm_guarantee}

The last missing piece of the PSCA algorithm requires verifying that the leader's strategy in the DSE $(\vtheta^*,\vmu^*)$ obtained from the GDA procedure is actually a $\sigma$-compression of $\Theta$ in $\cmp$. In principle, we should verify that $\min_{k \in [K]} D_2 (\vtheta^*_k, \vmu) \leq \sigma$ for any $\vmu \in \Theta$, which is equivalent to controlling if $\max_{\vmu \in \Theta} f(\vtheta^*, \vmu) \leq \sigma$. Unfortunately, the follower's strategy $\vmu^*$ is only locally optimal. Thus, checking $f(\vtheta^*, \vmu^*) \leq \sigma$ is not sufficient, as the globally optimal follower's strategy might attain a greater value of $f$ than $\vmu^*$. Instead, we should check $\mathcal{Z}_{\vtheta^*} \leq \sigma$, where $\mathcal{Z}_{\vtheta^*}$ is given by
\begin{equation}
	\mathcal{Z}_{\vtheta^*} = \max_{\omega \in \Omega_\Theta} \min_{k \in [K]} \int_{\Sspace\Aspace}  \frac{ \big( \omega (s, a) \big)^2 }{  d_{\vtheta_k^*}^{sa} (s, a) } \de s \de a,
	\label{eq:qcqp_guarantee}
\end{equation}
which can be written as a quadratically constrained quadratic program (see Appendix~\ref{apx:qcqp_program}). It might come as no surprise that solving this problem is NP-hard. Indeed, this is equivalent to the problem~\eqref{eq:game} with a fixed leader's strategy $\vtheta^*$, but the objective $f(\vtheta^*, \vmu)$ is still non-concave \wrt $\vmu$.
Luckily, we can reformulate this NP-hard problem in the surrogate linear program (see Appendix~\ref{apx:lp_program}):
\begin{equation}
	\big( \overline{\mathcal{Z}}_{\vtheta^*} \big)^{-\frac{1}{2}} = \max_{\omega \in \Omega_\Theta} \min_{k \in [K]} \int_{\Sspace\Aspace}  \frac{ \omega (s, a) }{ \big( d_{\vtheta_k^*}^{sa} (s, a) \big)^{- \frac{1}{2}} } \de s \de a,
	\label{eq:lp_guarantee}
\end{equation}
where the value $ \overline{\mathcal{Z}}_{\vtheta^*} $ is a conservative approximation of $\mathcal{Z}_{\vtheta^*}$, as stated in the following theorem.
\begin{restatable}[]{theorem}{coverGuarantee}
	The value $ \overline{\mathcal{Z}}_{\vtheta^*} $ is an upper bound to the value $\mathcal{Z}_{\vtheta^*}$, \ie $ \overline{\mathcal{Z}}_{\vtheta^*}  \geq \mathcal{Z}_{\vtheta^*}, \forall \vtheta^* \in \Theta^K$.
\end{restatable}

\section{GUARANTEES OF RL WITH A COMPRESSED POLICY SPACE}
\label{sec:theory}
In the previous sections, we have motivated the pursuit of a compression $\cTheta$ of the original policy space $\Theta$ in the CMP $\cmp$ as a way to improve the computation and sample efficiency of solving RL tasks defined upon $\cmp$.
Since this compression procedure induces a loss, albeit bounded, in the expressive power of the policy space, it is worth investigating the performance guarantees that we have when addressing RL tasks with $\cTheta$. We first analyze \emph{policy evaluation} (Section~\ref{sec:theory_evaluation}) and then \emph{policy optimization} (Section~\ref{sec:theory_optimization}). The reported theoretical results mostly combine techniques from~\cite{metelli2018pois, papini2019optimist}.

\subsection{Policy Evaluation}
\label{sec:theory_evaluation}
In policy evaluation~\cite{sutton2018reinforcement}, we aim to estimate the performance $J (\vtheta)$ of a target policy $\vtheta \in \Theta$ through sampled interactions with an MDP $\mdp$. In our case, we can only draw samples with the policies in $\cTheta$, and we have to provide an off-policy estimate of $J (\vtheta)$ via importance sampling. Since for any target policy $\vtheta$ we are guaranteed to have a sampling policy $\vtheta' \in \cTheta$ such that $D_2 (d_{\vtheta}^{sa} || d_{\vtheta'}^{sa}) \leq \sigma$, by choosing a convenient sampling policy in $\cTheta$, we can enjoy the following guarantee on the error we make when evaluating any target policy $\vtheta \in \Theta$ in any MDP $\mdp$ one can build upon $\cmp$.
\begin{restatable}[Policy Evaluation Error]{theorem}{isEvaluation}
	\label{thr:is_evaluation}
	Let $\cTheta$ be a $\sigma$-compression of $\Theta$ in $\cmp$, let $\R$ be a reward function for $\cmp$ uniformly bounded by $\Rmax$, let $\vtheta \in \Theta$ be a target policy, and let $\delta \in (0, 1)$ be a confidence. There exists $\vtheta' \in \cTheta$ such that, given $N$ i.i.d. samples from $d_{\vtheta'}^{sa}$,\footnote{\label{note}One can generate a sample from $d_{\vtheta'}^{sa}$ by drawing $s_0 \sim \mu$ and then following the policy $\vtheta'$. At each step $t$, the state $s_t$ and action $a_t$ are accepted with probability $\gamma$, whereas the simulation ends with probability $1 - \gamma$ \cite{metelli2021safe}.} the error of the importance sampling evaluation of $J (\vtheta)$ in $\mdp$, \ie $$\widehat{J}_{IS} (\vtheta / \vtheta') = \frac{1}{(1 - \gamma) N} \sum_{n = 1}^N w_{\vtheta / \vtheta'} (s_n, a_n) \R (s_n, a_n),$$ is upper bounded with probability at least $1 - \delta$ as 
$$| J(\vtheta) - \widehat{J}_{IS} (\vtheta / \vtheta') | \leq \frac{\Rmax}{1 - \gamma} \sqrt{ \frac{\sigma}{\delta N}}.$$
\end{restatable}
Notably, given a budget of samples $N$, a confidence $\delta$, and a requirement on the evaluation error beforehand, we could select a proper $\sigma$ to build a $\sigma$-compression that meets the requirement in any policy evaluation task. However, choosing a sampling policy $\vtheta' \in \cTheta$ that is best suited for a given task might be non-trivial. Thus, one can instead take a batch of $N_k$ samples with each policy in $\cTheta$, and then perform the policy evaluation via Multiple Importance Sampling (MIS)~\cite{owen2013monte, papini2019optimist}.
\begin{restatable}[]{corollary}{misEvaluation}
	\label{thr:mis_evaluation}
	Let $\cTheta$ be a $\sigma$-compression of $\Theta$ in $\cmp$ such that $|\cTheta| = K$, let $\R$ be a reward function for $\cmp$ uniformly bounded by $\Rmax$, let $\vtheta \in \Theta$ be a target policy, and let $\delta \in (0, 1)$ be a confidence. Given $N_k$ i.i.d. samples from each $d_{\vtheta_k}^{sa}$, $\vtheta_k \in \cTheta$, the error of the multiple importance sampling evaluation of $J (\vtheta)$ in $\mdp$, \ie
\begin{align*}
&\widehat{J}_{MIS} (\vtheta / \vtheta_1, \ldots, \vtheta_K) = \\
&\frac{1}{(1 - \gamma)}  \sum_{k = 1}^K \sum_{n = 1}^{N_k} \frac{d_{\vtheta}^{sa} (s_{n, k}, a_{n, k})}{\sum_{j = 1}^{K} N_j d_{\vtheta_j}^{sa} (s_{n, k}, a_{n, k}) } \R (s_{n,k}, a_{n,k}),
\end{align*}
	 is upper bounded with probability at least $1 - \delta$ as 
	 $$| J(\vtheta) - \widehat{J}_{MIS} (\vtheta / \vtheta_1, \ldots, \vtheta_K) | \leq \frac{\Rmax}{1 - \gamma} \sqrt{ \frac{D_2 (d_{\vtheta}^{sa} || \Phi) }{ \delta N }}$$ 
	 where $N = \sum_{k = 1}^K N_k$ is the total number of samples and $\Phi = \sum_{k = 1}^K \frac{N_k}{N} d_{\vtheta_k}^{sa}$ is a finite mixture.
\end{restatable}
Thanks to the result in~\citep[][Theorem 1]{metelli2020importance}, in tabular MDPs the evaluation error of the MIS estimator is guaranteed to be lower than the one of the IS estimator of Theorem~\ref{thr:is_evaluation} (as long as $N_k \geq N$, where $N$ is the number of samples considered by the IS estimator).

\subsection{Policy Optimization}
\label{sec:theory_optimization}
In policy optimization (see Section~\ref{sec:preliminaries_policy_optimization}), we seek for the policy $\vtheta$ that maximizes $J (\vtheta)$ within a parametric policy space. In principle, we could look for the policy that maximizes the performance within the $\sigma$-compression $\cTheta$, which can be found efficiently with the OPTIMIST algorithm~\cite{papini2019optimist}. Especially, in this setting OPTIMIST yields constant regret for tabular MDPs~\cite{metelli2020randomist}, as the set $\cTheta$ is finite and it is composed of stochastic policies such that $\forall \vtheta, \vtheta' \in \cTheta, D_2 (d_{\vtheta}^{sa} || d_{\vtheta'}^{sa}) < \infty$. However, this optimal policy within $\cTheta$ might be sub-optimal \wrt the optimal policy within the original policy space $\Theta$. We can still upper bound this sub-optimality, as reported in the following theorem.
\begin{restatable}[Policy Optimization in $\cTheta$]{theorem}{policyOptimizationCompressed}
	\label{thr:optimization}
	Let $\cTheta$ be a $\sigma$-compression of $\Theta$ in $\cmp$, and let $\R$ be a reward function for $\cmp$ uniformly bounded by $\Rmax$. The policy $\vtheta^*_\sigma \in \argmax_{\vtheta \in \cTheta} J (\vtheta)$ is $\epsilon$-optimal for the MDP $\mdp$, where
$$\epsilon := | \max_{\vtheta \in \Theta} J (\vtheta) - J (\vtheta^*_\sigma) | \leq \frac{\Rmax}{1 - \gamma} \sqrt{ \log \sigma }. $$
\end{restatable}
Notably, the latter guarantee does not involve any estimation, and the policy $\vtheta^*$ can be obtained in a finite number of interactions.
Nonetheless, one can shrink the sub-optimality $\epsilon$, and without deteriorating the sample complexity, by coupling the OPTIMIST algorithm with an additional offline optimization procedure. The idea is to return the policy $\vtheta \in \Theta$ that maximizes the importance sampling evaluation obtained with the samples from the policies in $\cTheta$.
\begin{restatable}[Off-Policy Optimization in $\Theta$]{theorem}{policyOptimizationGlobal}
	Let $\cTheta$ be a $\sigma$-compression of $\Theta$ in $\cmp$ such that $|\cTheta| = K$, let $\R$ be a reward function for $\cmp$ uniformly bounded by $\Rmax$, and let $\delta \in (0, 1)$ be a confidence. Given $N_k$ samples from each  $d_{\vtheta_k}^{sa}$, $\vtheta_k \in \cTheta$, we can recover an $\epsilon$-optimal policy for $\mdp$ as
	\begin{align}
		&\big(\ \_ \ , \ \vtheta^{*}_{IS} \big) \in \argmax_{\vtheta_k \in \cTheta, \vtheta \in \Theta : D_2 (d_{\vtheta}^{sa} || d_{\vtheta_k}^{sa} )} \nonumber \\
		&\qquad \frac{1}{(1 - \gamma) N_k} \sum_{n = 1}^{N_k} w_{\vtheta / \vtheta_k} (s_n, a_n) \R (s_n, a_n),
		\label{eq:offline_optimization}
	\end{align}
	such that with probability at least $1 - \delta$ $$\epsilon := \big| \max_{\vtheta \in \Theta} J (\vtheta) - J (\vtheta^*_{IS}) \big| \leq \frac{\Rmax}{1 - \gamma} \sqrt{ 2 \sigma / N_k \delta }.$$ 
\end{restatable}
Although, contrary to the guarantee in Theorem~\ref{thr:optimization}, $\epsilon$ vanishes with the number of samples in the latter result, solving the offline problem~\eqref{eq:offline_optimization} is non-trivial in general, as the policy space $\Theta$ is often infinite.

\section{NUMERICAL VALIDATION}
\label{sec:numerical_validation}
In this section, we provide a brief numerical validation of  the policy space compression problem (Section~\ref{sec:exp_policy_space_compression}) and how it benefits RL (Section~\ref{sec:exp_policy_evaluation},~\ref{sec:exp_policy_optimization}). 
To the purpose of the analysis, we consider the \emph{River Swim} domain~\cite{strehl2008analysis}, in which an agent navigates a chain of six states by taking one of two actions: either \emph{swim up}, to move upstream towards the upper states, or \emph{swim down}, to go downstream back to the lower states. 
Swimming upstream is harder than swimming downstream, thus the action \emph{swim up} fails with a positive probability, such that only a sequence of \emph{swim up} is likely to lead to the final state (an illustration of the corresponding CMP is reported in Figure~\ref{fig:river_swim_mdp_text}). 
In Appendix~\ref{apx:numerical_validation_details}, we report further details on the experimental settings, along with some additional results in a \emph{Grid World} environment. We leave as future work a more extensive experimental evaluation of the policy space compression problem beyond toy domains.
\begin{figure*}[t]
	\begin{subfigure}[t]{\textwidth}
    		\centering
    		\includegraphics[scale=1]{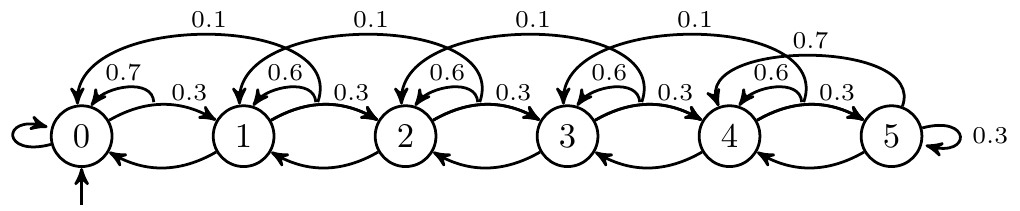}
    		\vspace{-0.2cm}
    		\caption{River Swim}
    		\label{fig:river_swim_mdp_text}
    	\end{subfigure}

	\begin{subfigure}[t]{0.6\textwidth}
    		\centering
    		\includegraphics[scale=1, valign=t]{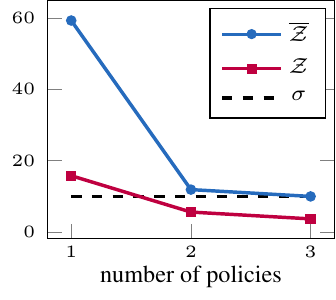}
    		\includegraphics[scale=1, valign=t]{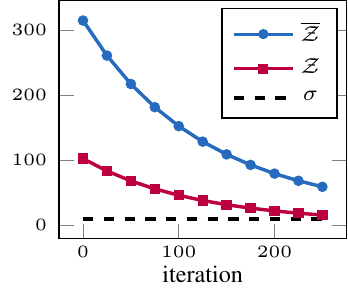}
    		\vspace{-0.2cm}
    		\caption{Policy Space Compression}
    		\label{fig:compression}
    	\end{subfigure}
    	\begin{subfigure}[t]{0.38\textwidth}
    		\centering
    		\includegraphics[scale=1, valign=t]{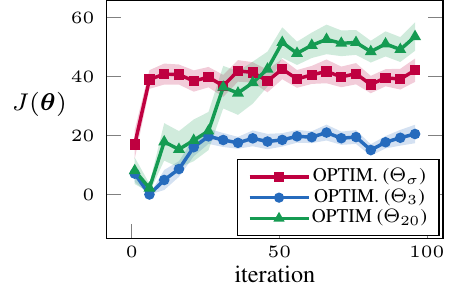}
    		\vspace{-0.2cm}
    		\caption{Policy Optimization}
    		\label{fig:optimization}
    	\end{subfigure}
    	
    	\vspace{0.1cm}
    	\begin{subfigure}[t]{0.54\textwidth}
    		\centering
    		\includegraphics[scale=1, valign=t]{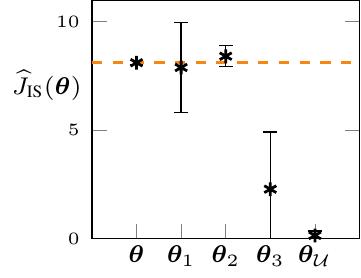}
    		\includegraphics[scale=1, valign=t]{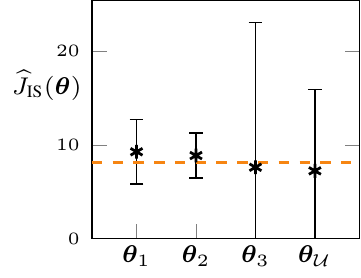}
    		\vspace{-0.15cm}
    		\caption{IS Policy Evaluation}
    		\label{fig:evaluation_is}
    	\end{subfigure}
    	\hfill
    	\begin{subfigure}[t]{0.45\textwidth}
    		\centering
    		\includegraphics[scale=1, valign=t]{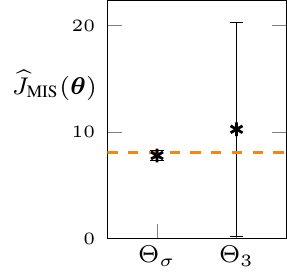}
    		\hspace{-0.2cm}
    		\includegraphics[scale=1, valign=t]{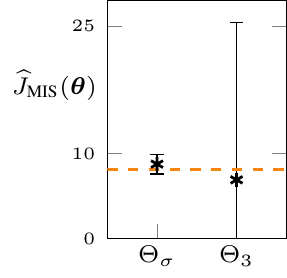}
    		\vspace{-0.15cm}
    		\caption{MIS Policy Evaluation}
    		\label{fig:evaluation_mis}
    	\end{subfigure}
    \caption{Set of experiments in the \emph{River Swim} domain, which is illustrated in \textbf{(a)}. \textbf{(b)} The value of the  compression guarantee $\mathcal{Z}$, its upper bound $\overline{\mathcal{Z}}$, and the requirement $\sigma$ as a function of the number of policies $K$ (left) and as a function of the iterations with $K = 1$ (right) obtained with PSCA. \textbf{(c)} The average return $J (\vtheta)$ obtained by OPTIMIST with the $\sigma$-compression $\cTheta$ (3 policies), a 3-policies discretization $\Theta_3$, and a 20-policies discretization $\Theta_{20}$ (95\% c.i. over 50 runs). \textbf{(d,e)} IS and MIS evaluation of $J(\vtheta)$ by taking samples with $\vtheta$, $\vtheta_k \in \cTheta$, a uniform policy $\vtheta_{\mathcal{U}}$, the mixture $\cTheta$, or a mixture of 3 random policies $\Theta_3$. We provide both the empirical (left, 95\% c.i. over 50 runs) and the hindsight (right) values.}
    \label{fig:numerical_validation}
\end{figure*}

\subsection{Policy Space Compression}
\label{sec:exp_policy_space_compression}
In the River Swim, we consider the policy space $\Theta \subseteq \Reals^{|\Sspace| \times (|\Aspace| - 1)}$ of the \emph{softmax} policies $\pi_{\vtheta} (a | s) = \exp(\theta_{sa}) / \sum_{j \in \Aspace} \exp(\theta_{sj})$, and we seek for a compression $\cTheta$ with the requirement $\sigma = 10$, such that $\cTheta$ is a valid $\sigma$-compression if $\min_{\vtheta \in \cTheta} \max_{\vmu \in \Theta} f (\vtheta, \vmu) \leq 10$. In Figure~\ref{fig:compression}, we report the values of $\mathcal{Z} = \max_{\vmu \in \Theta} f(\vtheta, \vmu)$~\eqref{eq:qcqp_guarantee} and its upper bound $\overline{\mathcal{Z}} \geq \mathcal{Z}$~\eqref{eq:lp_guarantee}. Especially, we can see that PSCA effectively found a valid $\sigma$-compression $\cTheta$ of just $K = 3$ policies (Figure~\ref{fig:compression}, left), and that the values of $\mathcal{Z}$ and $\overline{\mathcal{Z}}$ smoothly decreases during the GDA procedure for a fixed number of policies (Figure~\ref{fig:compression}, right). Notably, $K = 2$ policies are actually sufficient to meet the $\sigma$ requirement in this setting. However, PSCA cannot access $\mathcal{Z}$ but its conservative approximation $\overline{\mathcal{Z}}$, and thus stops whenever $\overline{\mathcal{Z}} \leq \sigma$. In Appendix~\ref{apx:numerical_validation_details}, we report an illustration of the obtained policies $\vtheta_k \in \cTheta$. This set coarsely includes two policies that swims up most of the time, either mixing the actions when the rightmost state is reached ($\vtheta_1$) or swimming up there as well ($\vtheta_2$), and a policy that swims down in the leftmost state and swims up in the others ($\vtheta_3$).

\subsection{Policy Evaluation with a Compressed Policy Space}
\label{sec:exp_policy_evaluation}
We now show that the obtained $\sigma$-compression $\cTheta$ can be employed with benefit in the most challenging policy evaluation task one can define in the River Swim, which is the off-policy evaluation of an $\epsilon$-greedy policy $\vtheta$ for the reward function that assigns $\Rmax = 100$ for taking the action \emph{swim up} in the rightmost state. In Figure~\ref{fig:evaluation_is}, we show that sampling with the policies $\vtheta_1, \vtheta_2 \in \cTheta$ lead to an IS off-policy evaluation that is comparable to the exact $J (\vtheta)$ (dashed line) and its on-policy estimate ($\vtheta$). Instead, the policy $\vtheta_3$ and a uniform policy $\vtheta_{\mathcal{U}}$ lead to significantly worse evaluations, as they collect too many samples in the leftmost state. Even by sampling from a uniform mixture of the policies in $\cTheta$, the performance of the MIS evaluation is significantly better than the one obtained by a uniform mixture of three random policies ($\Theta_3$), as reported in Figure~\ref{fig:evaluation_mis}. For both the IS and the MIS regime, we provide the empirical evaluations (on the left) and the hindsight evaluations (right) obtained with the exact values of the importance weights $w_{\vtheta / \vtheta'}$ and the confidence bounds of the Theorem~\ref{thr:is_evaluation},~\ref{thr:mis_evaluation} respectively.

\subsection{Policy Optimization with a Compressed Policy Space}
\label{sec:exp_policy_optimization}
Finally, we show that the compression $\cTheta$ allows for efficient policy optimization. We consider the same reward function of the previous section, and the OPTIMIST~\cite{papini2019optimist} algorithm equipped with $\cTheta$, or a uniform discretization of the original policy space $\Theta$ with either three policies ($\Theta_3$) or twenty policies ($\Theta_{20}$). In Figure~\ref{fig:optimization}, we show that OPTIMIST with $\cTheta$ swiftly converges (less than five iterations) to the optimal policy within the space. Instead, the policy space $\Theta_3$ leads to a huge sub-optimality in the final performance, and OPTIMIST with $\Theta_{20}$ is way slower to converge to the optimal policy within the space.
These results are a testament of the ability of PSCA to incorporate the peculiar structure of the domain in a small set of representative policies $\cTheta$, and to allows for a remarkable balance between sample efficiency and sub-optimality in subsequent policy optimization.

\section{DISCUSSION AND CONCLUSION}
\label{sec:conclusions}
In this paper, we considered the problem of compressing an infinite parametric policy space into a finite set of representative policies for a given environment. First, we provided a formal definition of the problem, and we highlighted its inherent hardness. Then, we proposed a tractable game-theoretic reformulation, for which a locally optimal solution can be efficiently found through an iterative GDA procedure. Finally, we provided a theoretical characterization of the guarantees that the compression brings to subsequent RL tasks, and a numerical validation of the approach.

\subsection{Related Works}

Previous works~\cite{gregor2016vic, eysenbach2018diayn, achiam2018variational, hansen2019visr} have considered heuristic methods to extract a convenient set of policies from the policy space, but they lack the formalization and the theoretical guarantees that we provided. 
Especially, \citet{eysenbach2021information} argue that the set of policies learned by those methods cannot be used to solve all the relevant policy optimization tasks. 
Those policies should be generally intended as effective initializations for subsequent adaptation procedures, operating in the original policy space once the task is revealed, rather than a minimal set of sufficient policies.
To the best of our knowledge, the only other work considering a formal criterion to operate a selection of the policies is~\cite{zahavy2021discovering}. Having some similarieties, our work and \cite{zahavy2021discovering} still differ for some crucial aspects. Whereas they look for a set of policies that maximizes the performance under the worst-case reward, we look for a set of policies that guarantees $\epsilon$-optimality for any task. They do not consider the parameterization of the policy space as an additional source of structure, and thus they do not fully exploit the interplay between the policy space and the environment as we do. Their problem formulation is multi-task, as they restrict the class of rewards to linear combinations of a feature vector, our formulation is instead fully reward-free. Overall, our policy space compression problem is more general, as it is solving the problem in~\cite{zahavy2021discovering} as a by-product. However, their problem might be easier in nature,\footnote{This is purely speculative as~\cite{zahavy2021discovering} does not provide a formal study of the computational complexity of the problem.} and thus preferable if one only cares about the worst-case performance.
Finally, \citet{eysenbach2021information} provide interesting insights on the information geometry of the space of the state distributions induced by a policy in a CMP, which can lead to compelling geometric interpretations of our policy space compression problem.

\subsection{Limitations and Future Directions}

The main limitation of our work is that the proposed algorithm is assuming full knowledge of the environment, which is uncommon in RL literature. However, we believe that PSCA is providing a clear blueprint for future works that might target the compression problem from interactions with an \emph{unknown} environment, to pave the way for scalable policy space compression. Especially, such an extension would require sample-based estimates of the gradients~\eqref{eq:follower_gradient}, \eqref{eq:leader_gradient}, and the global guarantee~\eqref{eq:lp_guarantee}. Whereas estimating the gradients of state-action distributions is not an easy feat, previous works provide useful inspiration~\cite{morimura2010derivatives, schroecker2017state, schroecker2018generative}.
Similarly, sample-based estimates of~\eqref{eq:lp_guarantee} can take inspiration from approximate linear programming methods for MDPs~\cite{de2003linear, pazis2011non}.
Another potential limitation of the proposed approach is the memory complexity required to store the compression, in contrast to the compact representations of common policy spaces, such as a small set of basis functions or a neural network architecture. A future work might focus on compact representations for a given compression.
Other interesting future directions include an extension of the policy space compression problem to the parameter-based perspective~\cite{sehnke2008parameter, metelli2018pois,papini2019optimist}, and the development of policy optimization algorithms that are tailored to exploit a compression of the policy space.

\bibliography{biblio}

\clearpage
\onecolumn
\appendix

\section{Proofs}
\label{apx:proofs}
\subsection{Proofs of Section~\ref{sec:algorithm}}

\followerGradient*
\begin{proof}
	Let $\vtheta_k$ be the active leader's component, \ie $\vtheta_k \in \argmin_{\vtheta_i \in \vtheta} D_2 (d_{\vmu}^{sa} || d_{\vtheta_i}^{sa})$. We can compute the gradient of the objective $f (\vtheta, \vmu)$ \wrt $\vmu$ as
	\begin{align*}
		\nabla_{\vmu} f(\vtheta, \vmu) 
		&= \nabla_{\vmu} D_2 (d_{\vmu}^{sa} || d_{\vtheta_k}^{sa}) \\
		&= \nabla_{\vmu} \int_{\Sspace\Aspace} d_{\vtheta_k} ^{sa} (s, a) \bigg( \frac{ d_{\vmu}^{sa} (s,a) }{d_{\vtheta_k} ^{sa} (s, a)}  \bigg)^2 \de s \de a \\
		&= 2 \int_{\Sspace\Aspace} d_{\vtheta_k} ^{sa} (s, a) \bigg( \frac{ d_{\vmu}^{sa} (s,a) }{d_{\vtheta_k} ^{sa} (s, a)}  \bigg)^2  \nabla_{\vmu} \log d_{\vmu}^{sa} (s, a) \de s \de a.
	\end{align*}
\end{proof}

\leaderGradient*
\begin{proof}
	We can compute the gradient of the objective $f (\vtheta, \vmu)$ \wrt $\vtheta_k \in \vtheta$ as
	\begin{align*}
		\nabla_{\vtheta_k} f(\vtheta, \vmu) 
		&= \nabla_{\vtheta_k} D_2 (d_{\vmu}^{sa} || d_{\vtheta_k}^{sa}) \\
		&= \nabla_{\vtheta_k} \int_{\Sspace\Aspace} d_{\vtheta_k} ^{sa} (s, a) \bigg( \frac{ d_{\vmu}^{sa} (s,a) }{d_{\vtheta_k} ^{sa} (s, a)}  \bigg)^2 \de s \de a \\
		&= - \int_{\Sspace\Aspace} d_{\vtheta_k} ^{sa} (s, a) \bigg( \frac{ d_{\vmu}^{sa} (s,a) }{d_{\vtheta_k} ^{sa} (s, a)}  \bigg)^2  \nabla_{\vtheta_k} \log d_{\vtheta_k}^{sa} (s, a) \de s \de a.
	\end{align*}
\end{proof}

\coverGuarantee*
\begin{proof}
	The result is straightforward from
	\begin{align*}
		\overline{\mathcal{Z}}_{\vtheta^*} = \big( \big( \overline{\mathcal{Z}}_{\vtheta^*} \big)^{-\frac{1}{2}} \big)^2  &= \max_{\omega \in \Omega_\Theta} \min_{k \in [K]} \bigg( \int_{\Sspace\Aspace}  \omega (s, a)  \big( d_{\vtheta_k^*}^{sa} (s, a) \big)^{-\frac{1}{2}} \de s \de a \bigg)^2 \\
		&\geq \max_{\omega \in \Omega_\Theta} \min_{k \in [K]} \int_{\Sspace\Aspace} \Big( \omega (s, a)  \big( d_{\vtheta_k^*}^{sa} (s, a) \big)^{-\frac{1}{2}} \Big)^2 \de s \de a 
		= \mathcal{Z}_{\vtheta^*}.
	\end{align*}
\end{proof}

\subsection{Proofs of Section~\ref{sec:theory}}

\begin{lemma}[Variance of the IS Estimator]
\label{thr:is_variance}
	Let $\cmp$ be a CMP, and let $\vtheta \in \Theta$ be a target policy. Let $\{ s_n, a_n \}_{n = 1}^{N}$ be a sample of state-action pairs taken with the policy $\vtheta'$ in $\cmp$. Then, the variance of the importance sampling evaluation of $J (\vtheta)$ in $\cmp$, \ie $\widehat{J}_{IS} (\vtheta / \vtheta') = \frac{1}{(1 - \gamma) N} \sum_{n = 1}^N w_{\vtheta / \vtheta'} (s_n, a_n) \R (s_n, a_n)$, can be upper bounded as
	\begin{equation*}
		\Var_{(s,a) \sim d_{\vtheta'}^{sa}} \big[ \widehat{J}_{IS} (\vtheta / \vtheta') \big] \leq \frac{ (\Rmax)^2  D_2 (d_{\vtheta}^{sa} || d_{\vtheta'}^{sa}) }{ (1 - \gamma)^2 \ N }.
	\end{equation*}
\end{lemma}
\begin{proof}
	The proof follows the derivation in~\citep[][Lemma 4.1]{metelli2018pois}. When considering state-action pairs (as opposed to trajectories in~\cite{metelli2018pois}) one should account for the dependency between state-actions in the same trajectory. Here we consider a batch of $N$ i.i.d. samples taken with the discounted state distribution $d_{\vtheta'}^{sa}$, in which the dependency vanishes. Especially, we write
	\begin{align*}
			\Var_{(s,a) \sim d_{\vtheta'}^{sa}} \big[ \widehat{J}_{IS} (\vtheta / \vtheta') \big] 
			&\leq \frac{1}{(1 - \gamma)^2 N} \Var_{(s,a) \sim d_{\vtheta'}^{sa}} \big[ w_{\vtheta / \vtheta'} (s, a) \R (s, a) \big] \\
			&\leq \frac{1}{(1 - \gamma)^2 N} \EV_{(s,a) \sim d_{\vtheta'}^{sa}} \bigg[ \bigg( \frac{d_{\vtheta}^{sa} (s,a)}{d_{\vtheta'}^{sa} (s,a)} \R (s, a) \bigg)^2 \bigg] \\
			&\leq \frac{(\Rmax)^2}{(1 - \gamma)^2 N} \EV_{(s,a) \sim d_{\vtheta'}^{sa}} \bigg[ \bigg( \frac{d_{\vtheta}^{sa} (s,a)}{d_{\vtheta'}^{sa} (s,a)}\bigg)^2 \bigg] 
			= \frac{ (\Rmax)^2  D_2 (d_{\vtheta}^{sa} || d_{\vtheta'}^{sa}) }{ (1 - \gamma)^2 \ N }.
	\end{align*}
	Note that the sampling procedure from a discounted state distribution is wasteful, as one should draw exactly $N$ trajectories from the discounted CMP $\cmp$ to collect just $N$ i.i.d. samples, while the other samples in the trajectories are discarded~\cite{metelli2021safe}. Nonetheless, one could refine this result to account for dependent data, by either exploiting the Bellman equation of the variance~\citep[see][]{sobel1982variance, xie2019towards} or concentration inequalities for Markov chains~\cite{paulin2015concentration}, which allows to upper bound the variance of the estimate computed over $N$ dependent samples from $d_{\vtheta'}^{sa}$.
\end{proof}

\isEvaluation*
\begin{proof}
	We would like to bound the difference $| J (\vtheta) - \widehat{J}_{IS} (\vtheta / \vtheta') |$ for a policy $\vtheta' \in \cTheta$. By the definition of $\sigma$-compression, there exists at least a policy $\vtheta' \in \cTheta$ such that $D_2 (d_{\vtheta}^{sa} || d_{\vtheta'}^{sa}) \leq \sigma$. Since the IS estimator $\widehat{J}_{IS} (\vtheta / \vtheta')$ is unbiased, and $\Var_{(s,a) \sim d_{\vtheta'}^{sa}} \big[ \widehat{J}_{IS} (\vtheta / \vtheta') \big] < \infty$ through Lemma~\ref{thr:is_variance}, we can use the Chebichev's inequality to write, $\forall \epsilon > 0$,
	\begin{equation*}
		Pr ( | J (\vtheta) - \widehat{J}_{IS} (\vtheta / \vtheta') | \geq \epsilon ) \leq \frac{\Var_{(s,a) \sim d_{\vtheta'}^{sa}} \big[ \widehat{J}_{IS} (\vtheta / \vtheta') \big]}{\epsilon^2}.
	\end{equation*}
	Then, by calling $\delta = \frac{\Var_{(s,a) \sim d_{\vtheta'}^{sa}} \big[ \widehat{J}_{IS} (\vtheta / \vtheta') \big]}{\epsilon^2}$ and considering the complimentary event, we get
	\begin{equation*}
		Pr \Big( | J (\vtheta) - \widehat{J}_{IS} (\vtheta / \vtheta') | \leq \frac{\Rmax}{1 - \gamma} \sqrt{\sigma / \delta N}  \Big) \geq 1 - \delta
	\end{equation*}
	where we upper bounded the variance of $\widehat{J}_{IS} (\vtheta / \vtheta')$ as in Lemma~\ref{thr:is_variance} and the R\'enyi $D_2 (d_{\vtheta}^{sa} || d_{\vtheta'}^{sa})$ with $\sigma$.
\end{proof}

\misEvaluation*
\begin{proof}
	Through the combination of~\citep[][Lemma 1]{papini2019optimist} and Lemma~\ref{thr:is_variance}, it is straightforward to derive
	\begin{equation}
		\Var_{(s,a) \sim d_{\vtheta_k}^{sa}} \big[ \widehat{J}_{MIS} (\vtheta / \vtheta_1, \ldots, \vtheta_K) \big] \leq \frac{ (\Rmax)^2  D_2 (d_{\vtheta}^{sa} || \Phi) }{ (1 - \gamma)^2 \ N }.
		\label{eq:mis_variance}
	\end{equation}
	Then, similarly as in Theorem~\ref{thr:is_evaluation}, we can use the Chebichev's inequality to write, $\forall \epsilon > 0$,
	\begin{equation*}
		Pr ( | J (\vtheta) - \widehat{J}_{MIS} (\vtheta / \vtheta_1, \ldots, \vtheta_K) | \geq \epsilon ) \leq \frac{\Var_{(s,a) \sim d_{\vtheta_k}^{sa}} \big[ \widehat{J}_{MIS} (\vtheta / \vtheta_1, \ldots, \vtheta_K) \big]}{\epsilon^2}.
	\end{equation*}
	By calling $\delta = \frac{\Var_{(s,a) \sim d_{\vtheta_k}^{sa}} \big[ \widehat{J}_{MIS} (\vtheta / \vtheta_1, \ldots, \vtheta_K) \big]}{\epsilon^2}$ and considering the complimentary event, we get
	\begin{equation*}
		Pr \Big( | J (\vtheta) - \widehat{J}_{MIS} (\vtheta / \vtheta_1, \ldots, \vtheta_K) | \leq \frac{\Rmax}{1 - \gamma} \sqrt{\frac{D_2 (d_{\vtheta}^{sa} || \Phi)}{ \delta N }}  \Big) \geq 1 - \delta
	\end{equation*}
	where we upper bounded the variance of $\widehat{J}_{MIS} (\vtheta / \vtheta_1, \ldots, \vtheta_K)$ as in~\eqref{eq:mis_variance}.
\end{proof}

\policyOptimizationCompressed*
\begin{proof}
	Let be $\vtheta^* \in \argmax_{\vtheta \in \Theta} J (\vtheta)$. From the definition of $\sigma$-compression we have that there exists at least a policy $\vtheta' \in \cTheta$ such that $D_2 ( d_{\vtheta^*}^{sa} || d_{\vtheta'}^{sa}) \leq \sigma$. Then, we can write
	\begin{align}
		(1 - \gamma) |J(\vtheta^*) - J(\vtheta')| 
		&= \bigg| \int_{\Sspace\Aspace} \R(s,a) \big( d_{\vtheta^*}^{sa} - d_{\vtheta'}^{sa} \big) \de s \de a \bigg| \label{eq:opt_1} \\
		&\leq \Rmax \int_{\Sspace\Aspace} \big| d_{\vtheta^*}^{sa} - d_{\vtheta'}^{sa} \big| \de s \de a \label{eq:opt_2}  \\
		&\leq \Rmax \sqrt{ d_{KL} (d_{\vtheta^*}^{sa} || d_{\vtheta'}^{sa}) } \label{eq:opt_3} \\
		&\leq \Rmax \sqrt{ \log \big( D_{2} (d_{\vtheta^*}^{sa} || d_{\vtheta'}^{sa}) \big) }
		= \Rmax \sqrt{\log \sigma} \label{eq:opt_4} 
	\end{align}
	where \eqref{eq:opt_1} is from the definition of $J$ given in Section~\ref{sec:preliminaries_policy_optimization}, \eqref{eq:opt_3} is obtained from \eqref{eq:opt_2} through the Pinsker's inequality, and \eqref{eq:opt_4} derives from $d_{KL} (p || q) = d_{1} (p || q) \leq d_2 (p || q) = D_{2} (p || q)$, which is straightforward from the definition of R\'enyi divergence. Finally, it is trivial to see that $J(\vtheta^*_\sigma) \geq J(\vtheta')$ for $\vtheta_{\sigma}^* \in \argmax_{\vtheta \in \cTheta} J(\vtheta)$.
\end{proof}

\policyOptimizationGlobal*
\begin{proof}
	Thanks to the definition of $\sigma$-compression and the guarantee provided by Theorem~\ref{thr:is_evaluation}, from the collected samples we have that there exists $\vtheta_k \in \cTheta$ such that
	\begin{equation*}
		\widehat{J}_{IS} (\vtheta / \vtheta_k) - \frac{\Rmax}{1 - \gamma} \sqrt{ \frac{2\sigma}{N_k \delta} } \leq J(\vtheta) \leq \widehat{J}_{IS} (\vtheta / \vtheta_k) + \frac{\Rmax}{1 - \gamma} \sqrt{ \frac{2\sigma}{N_k \delta} }
	\end{equation*}
	holds $\forall \vtheta \in \Theta$ with probability at least $1 - \delta /2$. Then, let $\vtheta_{IS}^*$ be a policy obtained as in~\eqref{eq:offline_optimization}, and let $\vtheta^* \in \argmax_{\vtheta \in \Theta} J(\vtheta)$. We consider the event in which $J (\vtheta_{IS}^*)$ falls below its lower confidence bound and $J(\vtheta^*)$ exceeds its upper confidence bound. It is easy to see that this event happens with probability at most $\delta$, whereas the complimentary event guarantees that
	\begin{equation*}
		|J(\vtheta^*) - J(\vtheta^*_{IS})| \leq \frac{\Rmax}{1 - \gamma} \sqrt{ \frac{2\sigma}{N_k \delta} }.
	\end{equation*}
\end{proof}

\section{Optimization Problems}
\label{apx:problems}
\subsection{Quadratic Program Formulation of~\eqref{eq:qcqp_guarantee}}
\label{apx:qcqp_program}
The optimization problem in~\eqref{eq:qcqp_guarantee} can be formulated into a quadratically constrained quadratic program as
\begin{equation*}
\begin{aligned}
    &\underset{z \in \Reals, \bm{\omega} \in \Reals^{\Sspace\Aspace}}{\text{maximize}} 
    && z \\
    & \ \text{subject to} 
    && z - \int_{\Sspace\Aspace} \frac{ \big( \omega (s,a) \big)^2}{d_{\vtheta^*_k}^{sa} (s,a)} \de s \de a \leq 0, \qquad \forall k \in [K] \\
    & & & \int_{\Aspace} \omega (s,a) \de a = (1 - \gamma) \mu (s) + \gamma \int_{\Sspace \Aspace} \omega (s', a') P(s | s', a') \de s' \de a',
    \qquad \forall s \in \Sspace \\
    & & & \omega (s,a) \geq 0, \qquad \forall s \in \Sspace, \forall a \in \Aspace.
\end{aligned}
\end{equation*}

\subsection{Linear Program Formulation of~\eqref{eq:lp_guarantee}}
\label{apx:lp_program}
The optimization problem in~\eqref{eq:lp_guarantee} can be formulated into a linear program as
\begin{equation*}
\begin{aligned}
    &\underset{z \in \Reals, \bm{\omega} \in \Reals^{\Sspace\Aspace}}{\text{maximize}} 
    && z \\
    & \ \text{subject to} 
    && z - \int_{\Sspace\Aspace} \frac{ \omega (s,a)}{d_{\vtheta^*_k}^{sa} (s,a)} \de s \de a \leq 0, \qquad \forall k \in [K] \\
    & & & \int_{\Aspace} \omega (s,a) \de a = (1 - \gamma) \mu (s) + \gamma \int_{\Sspace \Aspace} \omega (s', a') P(s | s', a') \de s' \de a',
    \qquad \forall s \in \Sspace \\
    & & & \omega (s,a) \geq 0, \qquad \forall s \in \Sspace, \forall a \in \Aspace.
\end{aligned}
\end{equation*}

\section{Further Details on the Numerical Validation}
\label{apx:numerical_validation_details}
In Section~\ref{sec:exp_policy_space_compression}, we commented the results of PSCA in the River Swim domain. For the sake of clarity, here we report an illustration of the River Swim CMP (Figure~\ref{fig:river_swim_mdp}), heatmap visualizations of the policies in the $\sigma$-compression obtained by PSCA (Figure~\ref{fig:river_0}-\ref{fig:river_2}), and the set of parameters we employed ($\sigma = 10, \alpha = 0.005, \beta = 0.1$). We further report the results of an additional policy space compression experiment in a Gridworld domain ($|\Sspace| = 9, |\Aspace| = 4$). In this setting, we considered $\sigma = 40, \alpha = 0.005, \beta = 0.1$, and the resulting $\sigma$-compression is composed of $K = 4$ policies (a visualization is provided in Figure~\ref{fig:gridworld_0}-\ref{fig:gridworld_3}).

In Section~\ref{sec:exp_policy_evaluation}, we reported a set of policy evaluation experiments in the River Swim domain. Especially, we considered an IS off-policy evaluation setting, in which we take a batch of samples with each policy $\vtheta_k \in \cTheta$, or with a uniform policy $\vtheta_{\mathcal{U}}$, or with the target policy itself $\vtheta$. For every policy, the batch is composed of $N = 100000$ samples, and it is obtained by drawing $5000$ trajectories of $20$ steps. Similarly, we considered a MIS off-policy evaluation setting, in which we take a batch of samples with the $\sigma$-compression $\cTheta$, or a set of three random policies $\Theta_3$. In both the cases, the batch is composed of $ N = 300000$ samples ($N_k = 100000$ for each policy in the space), obtained by drawing $15000$ trajectories of $20$ steps.

In Section~\ref{sec:exp_policy_optimization}, we reported a policy optimization experiment in the River Swim domain. To run this experiment, we implemented the action-based formulation of the OPTIMIST algorithm~\citep[][Algorithm 1]{papini2019optimist}. For each seed, we run the algorithm for $100$ iterations, in each iteration we collect $N = 1000$ samples, which are obtained from $50$ trajectories of $20$ steps. The value of the importance weights truncation $M$ and the confidence schedule $\delta_t$ are taken from the theoretical analysis in~\cite{papini2019optimist}.

\begin{figure*}[t]
	\begin{subfigure}[t]{\textwidth}
    		\centering
    		\includegraphics[scale=1.2, valign=t]{contents/plot_files/river_swim_mdp.pdf}
    		\vspace{-0.2cm}
    		\caption{River Swim}
    		\label{fig:river_swim_mdp}
    	\end{subfigure}
    	
    	\vspace{0.5cm}
    	\begin{subfigure}[t]{\textwidth}
    		\centering
    		\includegraphics[scale=0.38, valign=t]{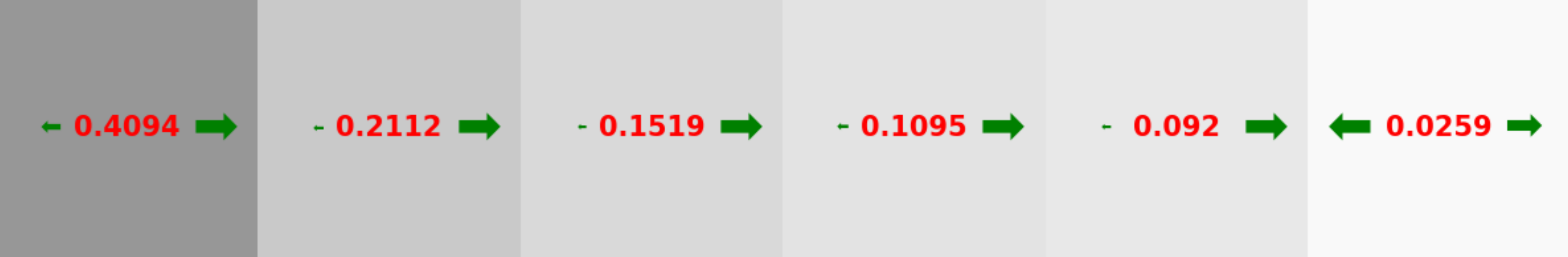}
    		\caption{$\vtheta_0$}
    		\label{fig:river_0}
    	\end{subfigure}
    	
    \vspace{0.3cm}
    	\begin{subfigure}[t]{\textwidth}
    		\centering
    		\includegraphics[scale=0.38, valign=t]{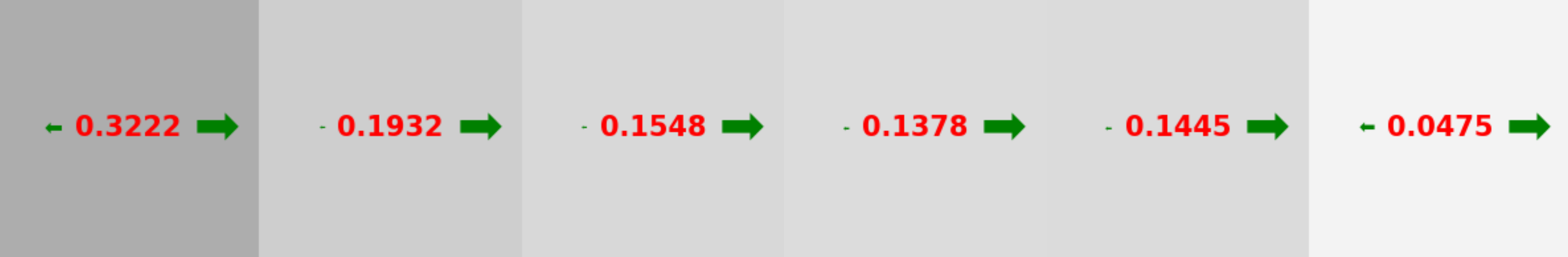}
    		\caption{$\vtheta_1$}
    		\label{fig:river_1}
    	\end{subfigure}
    	
    	\vspace{0.3cm}
    	\begin{subfigure}[t]{\textwidth}
    		\centering
    		\includegraphics[scale=0.38, valign=t]{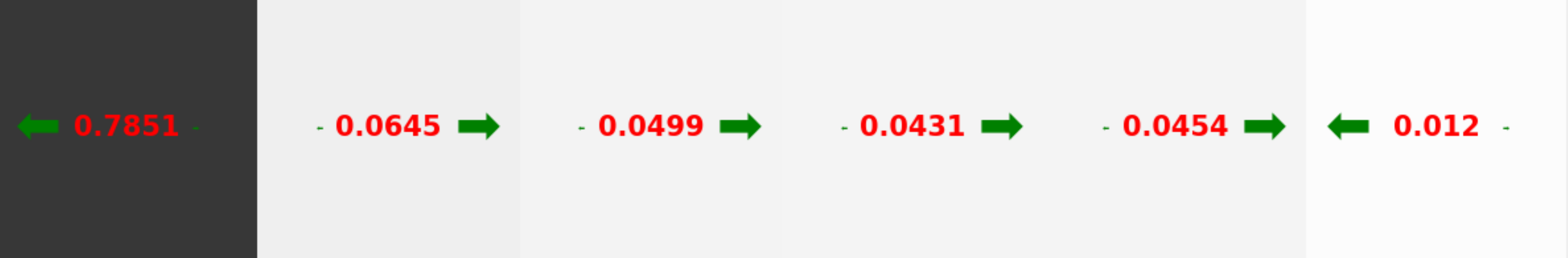}
    		\caption{$\vtheta_2$}
    		\label{fig:river_2}
    	\end{subfigure}
    	\vspace{0.3cm}
    \caption{\textbf{(a)} Illustration of the River Swim CMP. \textbf{(b, c, d)} Heatmap visualization of the policies in the $\sigma$-compression $\vtheta_k \in \cTheta$ for the River Swim domain. The background color and the label denote the state probability, the green arrows represent the policy in the state.}
    \label{fig:river_swim}
\end{figure*}
\begin{figure*}[t]
	\begin{subfigure}[t]{0.5\textwidth}
    		\centering
    		\includegraphics[scale=0.35, valign=t]{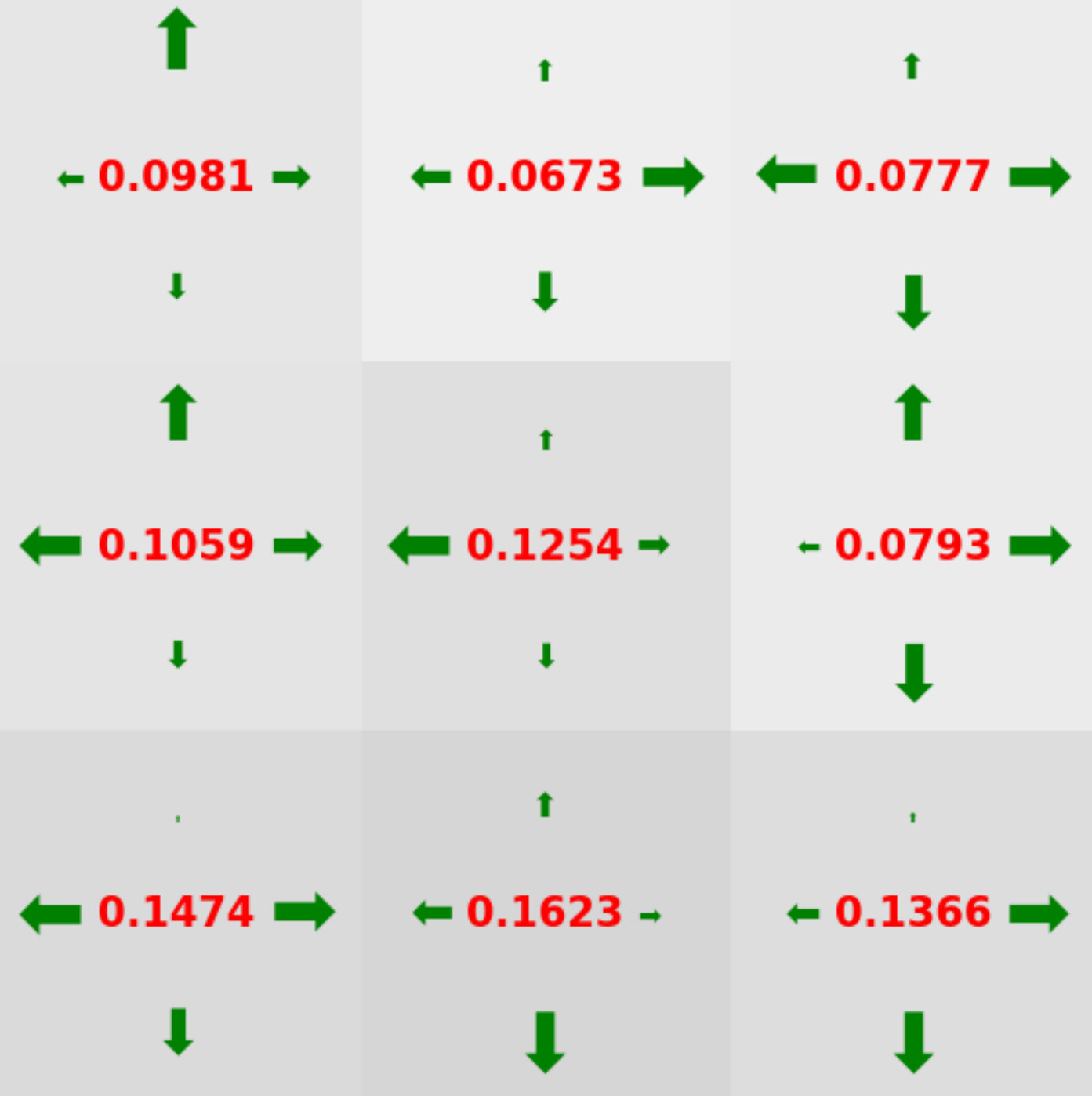}
    		\caption{$\vtheta_0$}
    		\label{fig:gridworld_0}
    	\end{subfigure}
    	\begin{subfigure}[t]{0.5\textwidth}
    		\centering
    		\includegraphics[scale=0.35, valign=t]{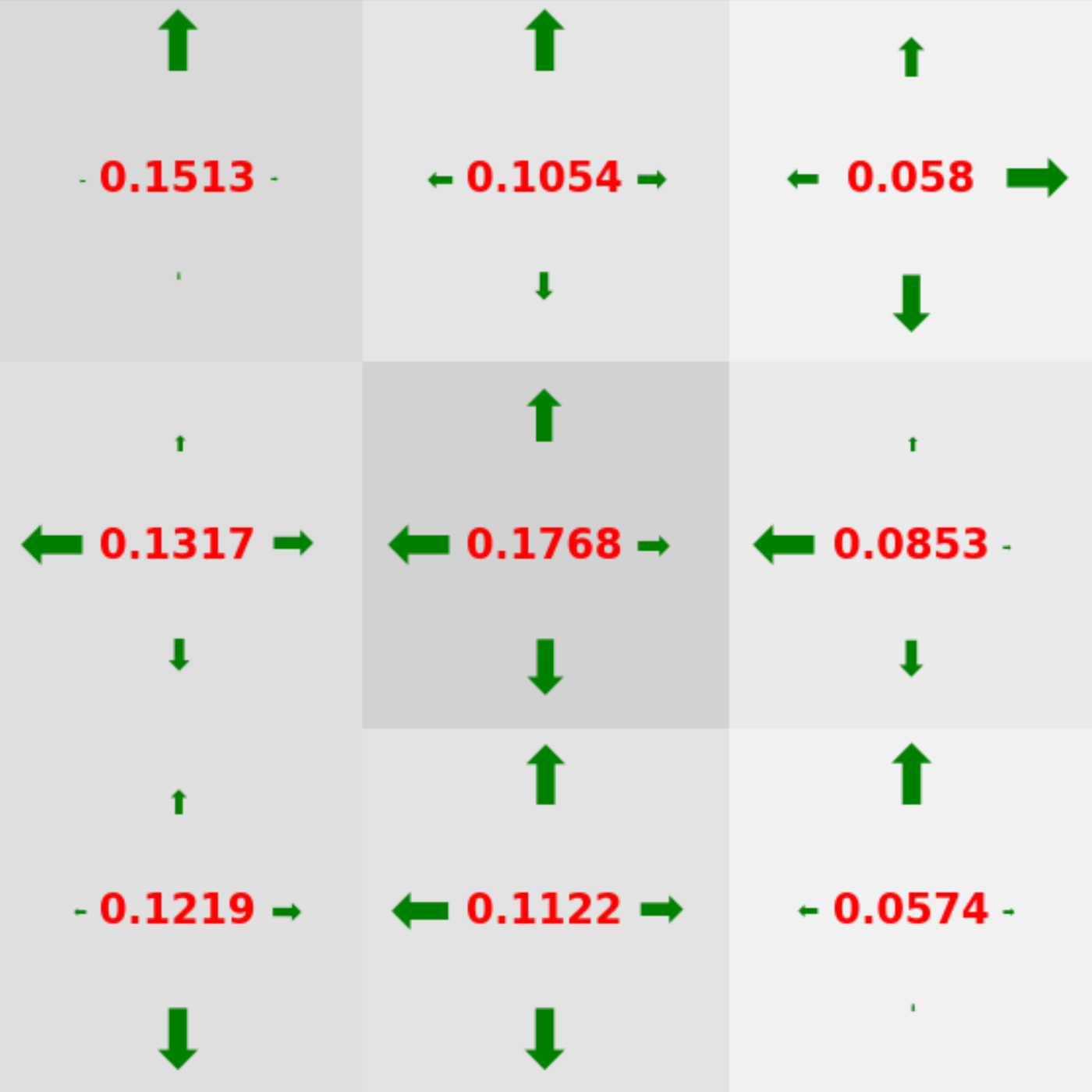}
    		\caption{$\vtheta_1$}
    		\label{fig:gridworld_1}
    	\end{subfigure}
    	
    	\vspace{0.3cm}
    	\begin{subfigure}[t]{0.5\textwidth}
    		\centering
    		\includegraphics[scale=0.35, valign=t]{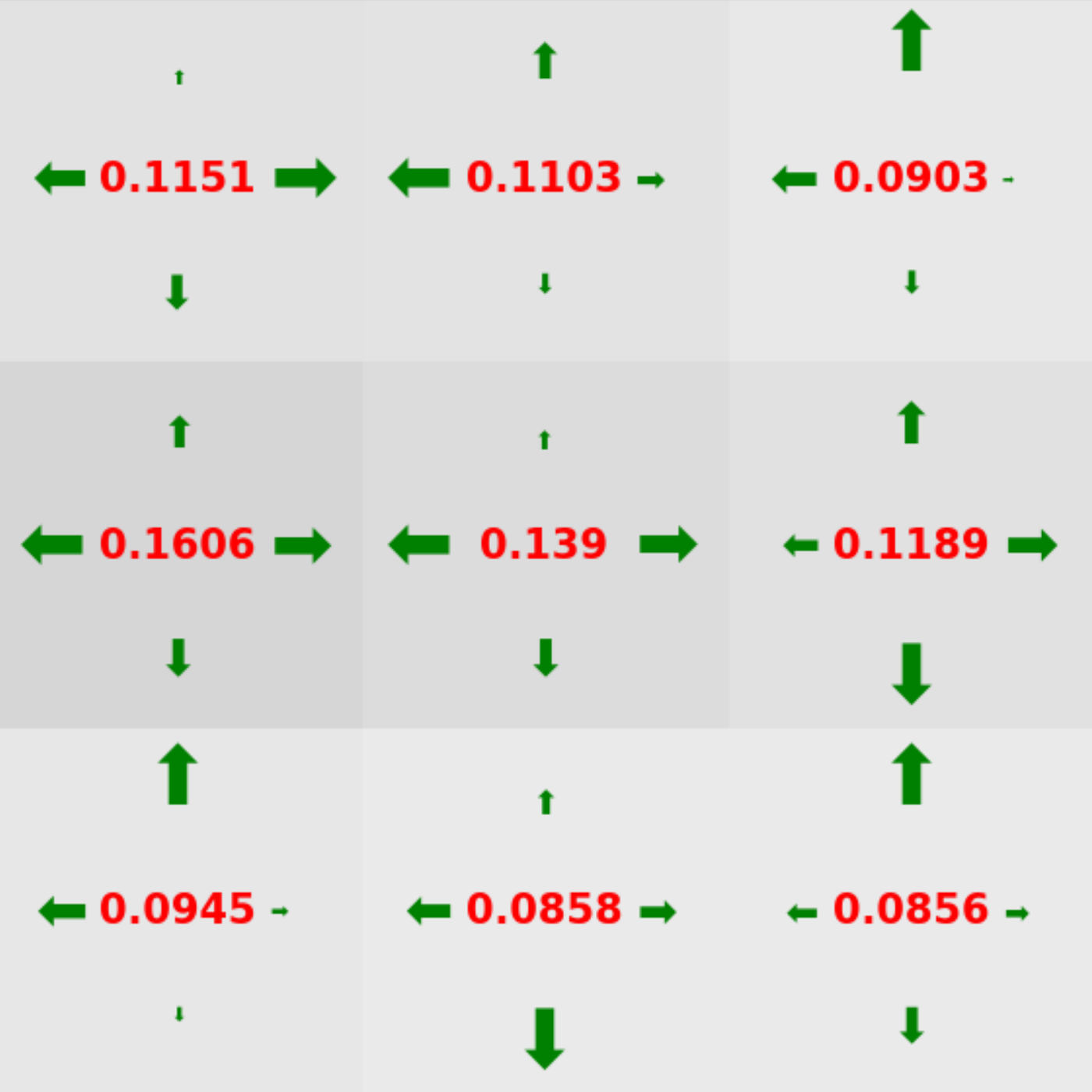}
    		\caption{$\vtheta_2$}
    		\label{fig:gridworld_2}
    	\end{subfigure}
    	\begin{subfigure}[t]{0.5\textwidth}
    		\centering
    		\includegraphics[scale=0.35, valign=t]{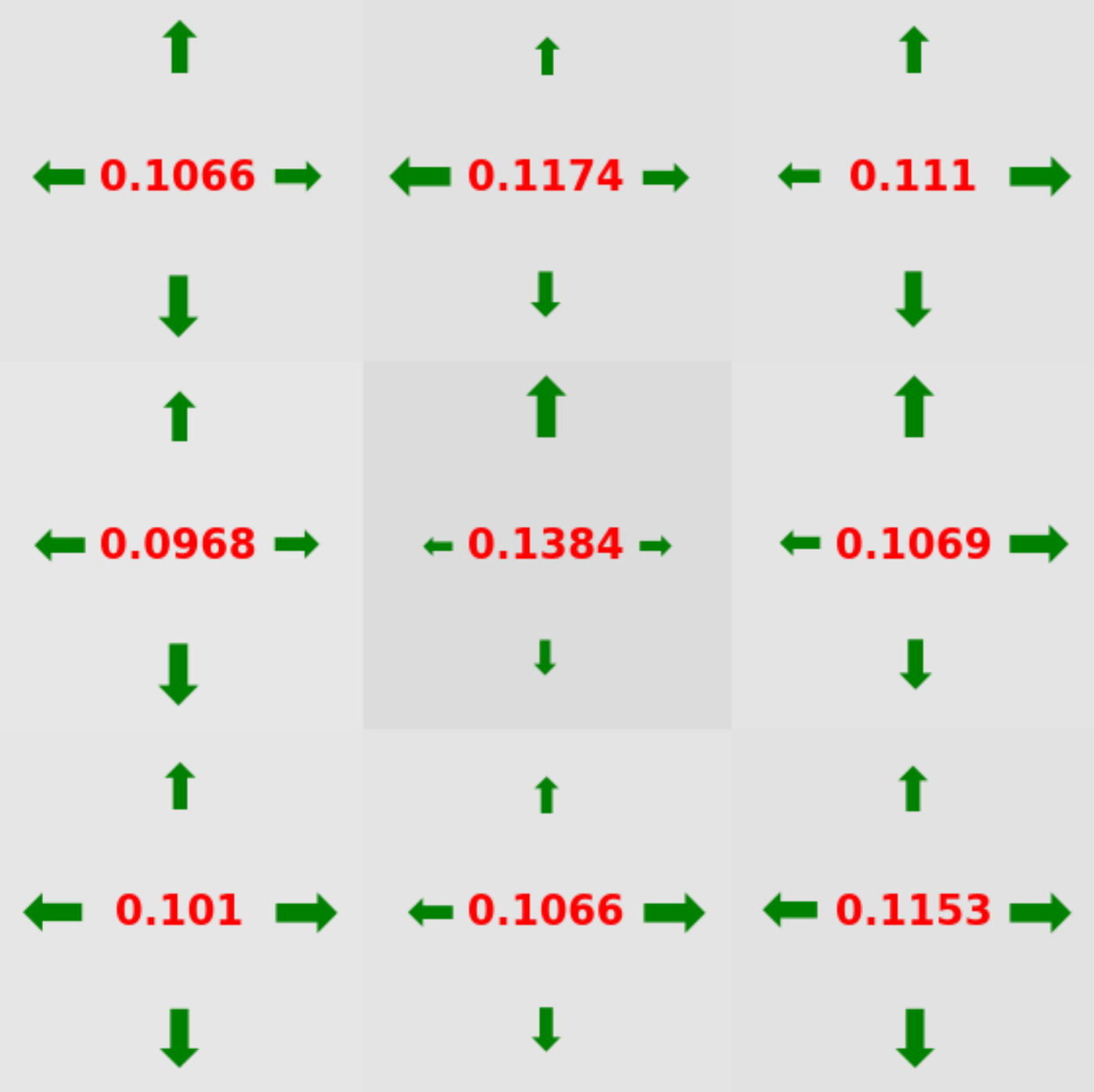}
    		\caption{$\vtheta_3$}
    		\label{fig:gridworld_3}
    	\end{subfigure}
    	\vspace{0.3cm}
    \caption{\textbf{(a, b, c, d)} Heatmap visualization of the policies in the $\sigma$-compression $\vtheta_k \in \cTheta$ for the Gridworld domain. The background color and the label denote the state probability, the green arrows represent the policy in the state.}
    \label{fig:gridworld}
\end{figure*}

\end{document}